\documentclass{article}
\usepackage{bbm,bm}
\usepackage{epsfig,ifthen}
\usepackage{latexsym}
\usepackage{amsfonts}
\usepackage{eepic}
\usepackage{amsopn}
\usepackage{amsmath,amssymb,amsthm,rotating,graphicx,rotating,float}
\usepackage{accents}
\usepackage[backref,pageanchor=true,plainpages=false,
pdfpagelabels,bookmarks,bookmarksnumbered,breaklinks,
pdfborder={0 0 0},  %removes outlines around hyper links in online display
]{hyperref}
\usepackage{cite}
\usepackage[mathscr]{eucal}
\usepackage{url}
\usepackage{graphicx}

\usepackage{ulem}
\usepackage{color} % for comments
\newcommand{\pen}[0]{p}
\newlength{\dhatheight}

\newcommand{\conf}{\delta}

\newcommand{\X}{\mathcal X} % instance space
\renewcommand{\H}{\mathcal H} % hypothesis space
\newcommand{\ind}{\mathbbm{1}}

\newcommand{\F}{\mathcal{F}}
\newcommand{\nats}{\mathbb{N}} % natural numbers
\newcommand{\reals}{\mathbb{R}} % real numbers

\newcommand{\E}{\mathbb{E}}

\newcommand{\argmin}{\mathop{\rm argmin}}
\newcommand{\ignore}[1]{}
\newcommand{\todo}[1]{}
\newcommand{\oldstuff}[1]{}

\newtheorem{theorem}{Theorem}
\newtheorem{corollary}[theorem]{Corollary}

\newsavebox{\savepar}
\newenvironment{bigboxit}{\begin{center}\begin{lrbox}{\savepar}
\begin{minipage}[h]{3.15in}
\normalfont
\begin{flushleft}}
{\end{flushleft}\end{minipage}\end{lrbox}\fbox{\usebox{\savepar}}
\end{center}}

\makeatletter
\newcommand{\vast}{\bBigg@{3}}
\newcommand{\Vast}{\bBigg@{4}}
\makeatother

\newcommand{\eat}[1]{}

\title{Learning with Changing Features}

\author{Amit Dhurandhar\footnote{IBM Research}, Steve Hanneke\footnote{Princeton} and Liu Yang\footnote{Yale}}

\begin{document}
\maketitle

\begin{abstract}
  In this paper we study the setting where features are added or
  change interpretation over time, which has applications in multiple
  domains such as retail, manufacturing, finance. In particular, we propose an approach to provably determine the
  time instant from which the new/changed features start becoming
  relevant with respect to an output variable in an agnostic (supervised) learning setting. We also
  suggest an efficient version of our approach which has the same
  asymptotic performance. Moreover, our theory also applies when we
  have more than one such change point. Independent post analysis of a change point identified by our method for a large retailer revealed that it corresponded in time with certain unflattering news stories about a brand that resulted in the change in customer behavior. We also applied our method to data from an advanced manufacturing plant identifying the time instant from which downstream features became relevant.

To the best of our knowledge this is the first work that formally studies change point detection in a distribution independent agnostic setting, where the change point is based on the changing relationship between input and output.

\end{abstract}

\section{Introduction}

In domains such as advanced manufacturing which involve thousands to
tens of thousands of processing stages spread over months, measurements
that are taken to monitor the quality of the products are usually very
expensive. This is because not only do the measurement tools, which
have to be extremely high precision, cost millions of dollars; but
each measurement slows down the line resulting in significant loss of
throughput/productivity. Hence, unless these measurements provide an
accurate indication of product health they are simply an incurred
cost. Thus as new tools or old tools after
maintenance are added to the production line, which results in
additional/altered measurements, it is critically important to know
the time instant from when these tools significantly impact quality.
Usually manufacturers/engineers will randomly infrequently measure the products,
but a more intelligent strategy can potentially save them
billions of dollars as premeptive/corrective actions can be taken on
likely to be faulty/subpar products, which can be predicted with much
improved accuracy. Moreover, it is quite possible that there was a
process change at the detected time instant and hence many of the
recent products followed a different route (viz. new tools). Based on
this, the intelligent strategy could be used to recommend that for
products that follow the old route we do not need the measurements
corresponding to the new tools, while those in the new route we should
heavily sample these measurements. Such a dynamic policy, where we
decide how to distribute measurements for better prediction of product
quality, can significantly improve overall profitability. For instance
in chip manufacturing, typically 1\% increase in yield --
i.e. percentage of within spec chips -- is worth over a billion
dollars in revenue. Similar gains can be seen in other types of
manufacturing such as pharmaceutical or processed food industries.

Our solution is also applicable to other domains such as retail, finance,
document classification, sensor networks, where features are
added or change interpretation over time. For instance in finance,
different financial indicators may be added over time to better
determine the health of a deal. Moreover, change in the competitive
landscape can lead to different outcomes even with the same values for these indicators after a certain point in
time. Identifying this time instant can be extremely important in
adapting to the changing environment. A similar issue can be witnessed in document classification, where for instance recent documents with the word "tweet" may be incorrectly classified as those about birds rather than technology if the change point was not promptly identified.

In particular, we make the following contributions: 
In Section \ref{sec:theory}, we provide distribution independent
  excess risk guarantees for the problem of statistical regression
  with no assumptions on the nature of the change. Furthermore, our guarantees are not
significantly worse than settings with \emph{no} feature changes 
making our approach effective in adapting to the change. Our results in the
  main article are for spaces of finite pseudo-dimension, so that our
  analysis is applicable to commonly used empirical risk minimization
  methods (viz. generalized least squares, logistic regression, etc.).
  The supplemental material (section 7) extends the theory to cover regularized
  learning rules (viz. lasso, ridge, regularized logistic regression,
  etc.)  and function classes of unbounded pseudo-dimension
  (viz. kernel regression). This extension, which is conceptually based on similar ideas that we have in the main article, nicely generalizes our theory with the added complexity of reasoning about tail conditions on the noise. In Section \ref{sec:theory}, we also show how our analysis can
  be applied to provably identify not just a single but multiple
  change points.
In Section \ref{sec:eff}, we provide a more efficient version of
  our algorithm based on our analysis in section \ref{sec:theory} and
  related work, with essentially the same asymptotic performance as
  our original algorithm.
In Section \ref{sec:exp}, we perform experiments on synthetic data as well as two real industrial datasets. The first real dataset is from a large retailer. Post analysis by domain experts of a change point identified by our method for a certain brand of interest for this retailer revealed that certain unflattering news stories around that time tarnished the brand image resulting in this change. This was an independent justification for our identified change point and a potential confirmation that it was possibly not just noise. The second real dataset is from an advanced manufacturing plant. In this case we effectively identified the time instant from which measurements from a downstream process started becoming relevant. We report results with three base learning methods namely; SVM with RBF kernel, Logistic regression and L1-regularized logistic regression to showcase the fact that the performance of our detection methods is not limited to any specific learning technique.

\section{Related Work}

There has been recent work \cite{jmlr2014} studying a
dynamically changing environment in which multiple new features are
added at each stage of a multistage process. However, the goal there was to suggest an
efficient and accurate meta-algorithm to update an already existing
regression model with the added features. Their strategy is one
component in making our algorithm in section \ref{sec:eff} more efficient when learning and
detecting the change.  However, the aim of the present work overall is 
different in that we focus on adapting to the time index when the new/changed
features actually become relevant to the learning problem.

Another piece of closely related work is that of change point
detection which has been heavily studied in statistics
\cite{cpd}. Typically, the goal is to find an instant in a time series
from where the values/distribution are significantly different than
the prior period. There are many statistics (viz. CUSUM based,
stability based) that have been developed to detect this change. Some
of these ideas have also been extended to the multivariate setting
\cite{cpd2}. All of this work however is different from ours, as it is primarily unsupervised.

This topic has intimate connections to the general subject of concept
drift in statistical learning \cite{min_concept,cd}, where the
function to be learned varies over time.  It is also related to the
topic of distribution drift~\cite{dd}, in which the marginal
distribution over $x$ changes over time. This captures phenomena such
as the fact that certain terminology has a life cycle, so that terms
that were previously common (such as ``milliner'') are later much less
common, and vice versa (such as ``click'').  However, the present work
differs significantly from the prior literature on both concept drift
and distribution drift.  Indeed, our setting itself differs, in that
these prior works focus on prediction problems, where the objective is
to generalize to new unseen test points; in contrast, we are
interested in a data set where features maybe added/change over time, and the task is to estimate the
regression function at these given points.  Thus, we have access to
\emph{all} of the response variables when estimating the time index at
which the change occurs.  Furthermore, we are mainly interested in more of a sudden change described by varying of the features
to the estimated function, rather than some notion of gradual drift of a
target concept, or general changes to a data distribution.

Our setting is also different from the problem of domain adaptation \cite{da1}. The main distinction between our task
and traditional domain adaptation is that in our case the learner isn't aware of the time instant at which the new/changed features gain significance; indeed, this is our main challenge.

\section{Framework and Theoretical Analysis}
\label{sec:theory}

Suppose that $x_{1},\ldots,x_{m}$ are data points in a space $\X$, and
$Y_{1},\ldots,Y_{m}$ are independent $\reals$-valued random variables,
with range contained in $[-B,B]$ for some $B \in [1,\infty)$
%, and
denote $\eta_{i} = \E[Y_{i}]$. The restriction $B \geq 1$ is merely 
for convenience.  The results clearly also have implications for values bounded in a range $[-b,b]$ with $b \in (0,1)$,
obtained simply by multiplying every function $h \in \H_{1} \cup \H_{2}$ and response $Y_{t}$
by $1/b$, calculating the bound below for the case $B=1$,
and multiplying the resulting bound by $b^{2}$.
First, in this abstract version of our setting, we consider general
function classes $\H_{1}$ and $\H_{2}$, both containing functions
mapping $\X \to [-B,B]$.  
Following \cite{pollard:84,pollard:90,anthony:99}, we let
$p_{i}$ denote the pseudo-dimension of
$\H_{i}$ for each $i \in \{1,2\}$: that is, 
$p_{i}$ is the largest integer $p \in \nats$ such that 
$\exists (z_{1},y_{1}),\ldots,(z_{p},y_{p}) \in \X \times \reals$ for which 
the collection of $p$-dimensional binary vectors
$\{ (\ind[h(z_{1}) \leq y_{1}],\ldots,\ind[h(z_{p}) \leq y_{p}]) : h \in \H_{i} \}$ has cardinality $2^{p}$.
For instance, the set of \emph{linear} functions mapping an $r$-dimensional
representation of points $x \in \X$ to $\reals$ has pseudo-dimension $r+1$ (see e.g., \cite{anthony:99}),
and this remains an upper bound on the pseudo-dimension for any fixed monotone transformation
of such linear functions (such as is used in \emph{logistic regression}) \cite{anthony:99}.
Throughout this section, we suppose $m \geq \max\{p_{1},p_{2}\}$ and $\min\{p_{1},p_{2}\} > 0$.
\eat{The supplemental material (section 7) includes an extension of this theory to allow for spaces with 
infinite pseudo-dimension; the extension introduces additional technical complexity
to the arguments and theorems, but the essential ideas are basically the same as those 
used in the simpler case here.}

For any functions $h_{1} \in \H_{1}$ and $h_{2} \in \H_{2}$, 
and any $t_{0} \in \{1,\ldots,m+1\}$, denote
\begin{equation*}
\begin{split}
&R^{*}(h_{1},h_{2},t_{0})\\& = \frac{1}{m} \left( \sum_{t=1}^{t_{0}-1} ( h_{1}(x_{t}) - \eta_{t} )^{2} + \sum_{t=t_{0}}^{m} ( h_{2}(x_{t}) - \eta_{t} )^{2} \right)
\end{split}
\end{equation*}
and 
\begin{equation*}
\begin{split}
&\hat{R}(h_{1},h_{2},t_{0})\\& = \frac{1}{m} \left( \sum_{t=1}^{t_{0}-1} ( h_{1}(x_{t}) - Y_{t} )^{2} + \sum_{t=t_{0}}^{m} ( h_{2}(x_{t}) - Y_{t} )^{2} \right).
\end{split}
\end{equation*}
To be clear, in the cases of $t_{0} \in \{1,m+1\}$, we are defining $\sum_{t=1}^{0} \cdot = \sum_{t=m+1}^{m} \cdot = 0$ in these definitions (and below).
We then denote by 

$(h_{1}^{*},h_{2}^{*},t^{*}) = \argmin_{(h_{1},h_{2},t_{0}) \in \H_{1} \times \H_{2} \times \{1,\ldots,m+1\}} R^{*}(h_{1},h_{2},t_{0})$.
We are interested in obtaining $\hat{h}_{1} \in \H_{1}$, $\hat{h}_{2} \in \H_{2}$, and $\hat{t} \in \{1,\ldots,m+1\}$ based only on $Y_{1},\ldots,Y_{m}$,
such that $R^{*}(\hat{h}_{1},\hat{h}_{2},\hat{t})$ is not too much larger than $R^{*}(h_{1}^{*},h_{2}^{*},t^{*})$.

\begin{bigboxit}
In particular, let us choose
\begin{equation*}
(\hat{h}_{1},\hat{h}_{2},\hat{t}) = \argmin_{(h_{1},h_{2},t_{0}) \in \H_{1} \times \H_{2} \times \{1,\ldots,m+1\}} \hat{R}(h_{1},h_{2},t_{0}).
\end{equation*}
We refer to this general strategy as the {\bf Search-and-Split} algorithm (abbreviated as $\mathbf{SaS}$ below).
\end{bigboxit}

This method is essentially a variant of \emph{empirical risk minimization} for this setting 
in which there is an unknown change time.\eat{  The supplementary material includes an extension 
of our theory to allow for regularized learning rules as well.  This introduces some technical
complications, but the core argument remains the same as in the simpler case presented here.}
Specifically, we have the following theorem for the above Search-and-Split method.

\begin{theorem}
\label{thm:main}
With probability at least $1-\delta$, 
\begin{align*}
 &R^{*}(\hat{h}_{1},\hat{h}_{2},\hat{t})
\leq  R^{*}(h_{1}^{*},h_{2}^{*},t^{*}) \\
 &+ 22 B \sqrt{\frac{2 \ln(2 (m+1)/\delta) + \sum_{j=1}^{2} 3 p_{j} \ln( e m B / p_{j} )}{m}}.
\end{align*}
\end{theorem}
\begin{proof}
Let $\tilde{Y}_{1},\ldots,\tilde{Y}_{m}$ be equal in distribution to $Y_{1},\ldots,Y_{m}$ but independent of $Y_{1},\ldots,Y_{m}$,
and define $\tilde{R}(h_{1},h_{2},t_{0}) = \frac{1}{m} \left( \sum_{t=1}^{t_{0}-1} ( h_{1}(x_{t}) - \tilde{Y}_{t} )^{2} + \sum_{t=t_{0}}^{m} ( h_{2}(x_{t}) - \tilde{Y}_{t} )^{2} \right)$.
For any fixed choices of $h_{1}$, $h_{2}$, and $t \in \{1,\ldots,m+1\}$,
Hoeffding's inequality implies that with probability at least $1-\delta^{\prime}$, 
\begin{equation}
\label{eqn:hoeffding-basic}
\left| \hat{R}(h_{1},h_{2},t) - \E[ \tilde{R}(h_{1},h_{2},t) ] \right|
\leq \sqrt{\frac{2 B^{2} \ln(2/\delta^{\prime})}{m}}.
\end{equation}
Fix $\epsilon = \sqrt{\frac{\max\{p_{1},p_{2}\}}{m}}$.
Now, for each $j \in \{1,2\}$, let $\H_{j,\epsilon}$ denote a minimal $\epsilon$-cover of $\H_{j}$ with respect to the pseudo-metric 
$(h,g) \mapsto \rho_{m}(h,g) = \max_{1 \leq i \leq m} | h(x_{i}) - g(x_{i}) |$.
It is known (see e.g., \cite{anthony:99}) that 
\begin{equation}
\label{eqn:cover-size}
|\H_{j,\epsilon}| \leq (e m B / (\epsilon p_{j}))^{p_{j}}.
\end{equation}

By a union bound, \eqref{eqn:hoeffding-basic} holds simultaneously for every choice of 
$h_{1} \in \H_{1,\epsilon}$, $h_{2} \in \H_{2,\epsilon}$, and $t \in \{1,\ldots,m+1\}$,
with probability at least $1 - \delta^{\prime} |\H_{1,\epsilon}| |\H_{2,\epsilon}| (m+1)$.
Taking $\delta^{\prime} = \delta / ( |\H_{1,\epsilon}| |\H_{2,\epsilon}| (m+1) )$, this holds with probability at least $1-\delta$.
Also define $h_{j,\epsilon}^{*} = \argmin_{h \in \H_{j,\epsilon}} \rho_{m}(h,h_{j}^{*})$
and $\hat{h}_{j,\epsilon} = \argmin_{h \in \H_{j,\epsilon}} \rho_{m}(h,\hat{h}_{j})$, for each $j \in \{1,2\}$.

To simplify notation, denote $\hat{y}_{t} = \hat{h}_{1}(x_{t})$ and $\hat{y}_{t,\epsilon} = \hat{h}_{1,\epsilon}(x_{t})$ for $t \leq \hat{t}-1$,
and denote $\hat{y}_{t} = \hat{h}_{2}(x_{t})$ and $\hat{y}_{t,\epsilon} = \hat{h}_{2,\epsilon}(x_{t})$ for $t \geq \hat{t}$.
Similarly, denote $y_{t}^{*} = h_{1}^{*}(x_{t})$ and $y_{t,\epsilon}^{*} = h_{1,\epsilon}^{*}(x_{t})$ for $t \leq t^{*}-1$,
and denote $y_{t}^{*} = h_{2}^{*}(x_{t})$ and $y_{t,\epsilon}^{*} = h_{2,\epsilon}^{*}(x_{t})$ for $t \geq t^{*}$.
Then note that, by straightforward calculations,
\begin{align*}
&R^{*}(\hat{h}_{1},\hat{h}_{2},\hat{t}) - R^{*}(h_{1}^{*},h_{2}^{*},t^{*})\\ & = \E\left[ \tilde{R}(\hat{h}_{1},\hat{h}_{2},\hat{t}) | \hat{h}_{1},\hat{h}_{2},\hat{t} \right] - \E\left[ \tilde{R}(h_{1}^{*},h_{2}^{*},t^{*}) \right]
\\ & \leq \frac{1}{m} \sum_{t=1}^{m} \E\left[ ( \hat{y}_{t,\epsilon} - \tilde{Y}_{t} )^{2} + \epsilon^{2} + 2 \epsilon | \hat{y}_{t,\epsilon} - \tilde{Y}_{t} | \Big| \hat{h}_{1},\hat{h}_{2},\hat{t} \right]
\\ & - \frac{1}{m} \sum_{t=1}^{m} \E\left[ (y_{t,\epsilon}^{*} - \tilde{Y}_{t})^{2} - \epsilon^{2} - 2 \epsilon | y_{t,\epsilon}^{*} - \tilde{Y}_{t} | \right]
\\ & \leq 2 \epsilon^{2} + 8 \epsilon B 
+ \E\left[ \tilde{R}(\hat{h}_{1,\epsilon}, \hat{h}_{2,\epsilon}, \hat{t}) \Big| \hat{h}_{1}, \hat{h}_{2}, \hat{t} \right] - \E\left[ \tilde{R}(h_{1,\epsilon}^{*},h_{2,\epsilon}^{*},t^{*}) \right].
\end{align*}
Thus, on the above event of probability $1-\delta$, 
\begin{align*}
&R^{*}(\hat{h}_{1},\hat{h}_{2},\hat{t}) - R^{*}(h_{1}^{*},h_{2}^{*},t^{*})\\& \leq
2 \epsilon^{2} + 8 \epsilon B + 2 \sqrt{\frac{2 B^{2} \ln(2/\delta^{\prime})}{m}} 
+ \hat{R}(\hat{h}_{1,\epsilon},\hat{h}_{2,\epsilon},\hat{t}) - \hat{R}(h_{1,\epsilon}^{*},h_{2,\epsilon}^{*},t^{*}).
\end{align*}
Then note that 
\begin{align}
&\hat{R}(\hat{h}_{1,\epsilon},\hat{h}_{2,\epsilon},\hat{t}) - \hat{R}(h_{1,\epsilon}^{*},h_{2,\epsilon}^{*},t^{*}) \notag
 \\& \leq \frac{1}{m} \sum_{t=1}^{m} \left( ( \hat{y}_{t} - Y_{t} )^{2} + \epsilon^{2} + 2 \epsilon | \hat{y}_{t} - Y_{t} | \right) \notag
\\ & ~~- \frac{1}{m} \sum_{t=1}^{m} \left( (y_{t}^{*} - Y_{t})^{2} - \epsilon^{2} - 2 \epsilon | y_{t}^{*} - Y_{t} | \right) \notag
\\ & \leq 2 \epsilon^{2} + 8 \epsilon B + \hat{R}(\hat{h}_{1}, \hat{h}_{2}, \hat{t}) - \hat{R}(h_{1}^{*},h_{2}^{*},t^{*}) \notag
\\ & \leq  2 \epsilon^{2} + 8 \epsilon B. \label{eqn:empirical-diff-nonpositive}
\end{align}
Altogether, and combined with \eqref{eqn:cover-size}, we have that, with probability at least $1-\delta$, 
\begin{align*}
&R^{*}(\hat{h}_{1},\hat{h}_{2},\hat{t}) - R^{*}(h_{1}^{*},h_{2}^{*},t^{*})\\
& \leq 4 \epsilon^{2} + 16 \epsilon B + 2 \sqrt{\frac{2 B^{2} \ln(2/\delta^{\prime})}{m}}
\\ & \leq 20 B \sqrt{\frac{\max\{p_{1},p_{2}\}}{m}} 
\\ & + 2 B \sqrt{\frac{2 \ln(2(m+1)/\delta) + \sum_{j=1}^{2} 3 p_{j} \ln( e m B / p_{j} )}{m}}
\\ & \leq 22 B \sqrt{\frac{2 \ln(2(m+1)/\delta) + \sum_{j=1}^{2} 3 p_{j} \ln( e m B / p_{j} )}{m}}.
\end{align*}
\end{proof}

\paragraph{Application to Addition of New Features:}
Next, consider the special case in which there exist feature functions $\phi_{1},\ldots,\phi_{d+k} : \X \to \reals$, for $d,k \in \nats$,
and there exist function classes $\F_{1},\F_{2}$ such that every $f \in \F_{1}$ maps $\reals^{d} \to [-B,B]$,
while every $f \in \F_{2}$ maps $\reals^{d+k} \to [-B,B]$.
Then we can use the above framework to discuss the scenario in which the learner is tasked with 
identifying a time $t_{0}$ before which the first $d$ features suffice for good performance, and after which 
the full $d+k$ features are needed to obtain good performance.  Specifically, in this case, the class $\H_{1}$
is the set of functions $x \mapsto f(\phi_{1}(x),\ldots,\phi_{d}(x))$ s.t. $f \in \F_{1}$, and the class $\H_{2}$
is the set of functions $x \mapsto f(\phi_{1}(x),\ldots,\phi_{d+k}(x))$ s.t. $f \in \F_{2}$.
Then overloading the above notation, so that
\begin{align*}
\hat{R}(f_{1},f_{2},t_{0}) = &
\sum_{t=1}^{t_{0}-1} ( f_{1}(\phi_{1}(x_{t}),\ldots,\phi_{d}(x_{t})) - Y_{t} )^{2} \\&
+ \sum_{t=t_{0}}^{m} ( f_{2}(\phi_{1}(x_{t}),\ldots,\phi_{d+k}(x_{t})) - Y_{t} )^{2}
\end{align*}
and
\begin{align*}
R^{*}(f_{1},f_{2},t_{0}) = & 
\sum_{t=1}^{t_{0}-1} ( f_{1}(\phi_{1}(x_{t}),\ldots,\phi_{d}(x_{t})) - \eta_{t} )^{2} \\&
 + \sum_{t=t_{0}}^{m} ( f_{2}(\phi_{1}(x_{t}),\ldots,\phi_{d+k}(x_{t})) - \eta_{t} )^{2},
\end{align*}
consider the algorithm that chooses
\begin{equation*}
(\hat{f}_{1},\hat{f}_{2},\hat{t}) = \argmin_{(f_{1},f_{2},t_{0}) \in \F_{1} \times \F_{2} \times \{1,\ldots,m+1\}} \hat{R}(f_{1},f_{2},t_{0}).
%\sum_{t=1}^{t_{0}-1} ( f_{1}(\phi_{1}(x_{t}),\ldots,\phi_{d}(x_{t})) - Y_{t} )^{2} + \sum_{t=t_{0}}^{m} ( f_{2}(\phi_{1}(x_{t}),\ldots,\phi_{d+k}(x_{t})) - Y_{t} )^{2}.
\end{equation*}
Then for 
\begin{equation*}
(f_{1}^{*},f_{2}^{*},t^{*}) = \argmin_{(f_{1},f_{2},t_{0}) \in \F_{1} \times \F_{2} \times \{1,\ldots,m+1\}} R^{*}(f_{1},f_{2},t_{0}),
%\sum_{t=1}^{t_{0}-1} ( f_{1}(\phi_{1}(x_{t}),\ldots,\phi_{d}(x_{t})) - \eta_{t} )^{2} + \sum_{t=t_{0}}^{m} ( f_{2}(\phi_{1}(x_{t}),\ldots,\phi_{d+k}(x_{t})) - \eta_{t} )^{2},
\end{equation*}
Theorem~\ref{thm:main} implies the following corollary.

\begin{corollary}
\label{cor:main}
With probability at least $1-\delta$, 
\begin{align*}
%& \frac{1}{m} \left( \sum_{t=1}^{\hat{t}-1} ( \hat{f}_{1}(\phi_{1}(x_{t}),\ldots,\phi_{d}(x_{t})) - \eta_{t} )^{2} + \sum_{t=\hat{t}}^{m} ( \hat{f}_{2}(\phi_{1}(x_{t}),\ldots,\phi_{d+k}(x_{t})) - \eta_{t} )^{2} \right)
& R^{*}(\hat{f}_{1},\hat{f}_{2},\hat{t})
\leq R^{*}(f_{1}^{*},f_{2}^{*},t^{*})
\\ & + 22 B \sqrt{\frac{2 \ln(2 (m+1) /\delta) + \sum_{j=1}^{2} 3 p_{j} \ln( e m B / p_{j} )}{m}}. 
\end{align*}
\end{corollary}

\paragraph{Remarks on Computational Speedups:}
For the sake of reducing the computational burden of searching over values of $t_{0}$
to identify $\hat{t}$, we can alternatively search over a grid of values $i \lfloor \sqrt{m}/B \rfloor$, $i \in \nats$ with $i \lfloor \sqrt{m}/B \rfloor \leq m+1$
(supposing $B \leq \sqrt{m}$ for simplicity).
Denoting
\begin{equation*}
(\hat{g}_{1},\hat{g}_{2},\hat{t}_{\sqrt{m}}) = 
\!\!\!\!\!\!\!\!\argmin_{\substack{(f_{1},f_{2},t_{0}) \in \\ \F_{1} \times \F_{2} \times \{ i \lfloor \sqrt{m}/B \rfloor : i \lfloor \sqrt{m}/B \rfloor \leq m+1\}}}
\!\!\!\!\!\!\!\!\hat{R}(f_{1},f_{2},t_{0}),
%\sum_{t=1}^{t_{0}-1} ( f_{1}(\phi_{1}(x_{t}),\ldots,\phi_{d}(x_{t})) - Y_{t} )^{2} + \sum_{t=t_{0}}^{m} ( f_{2}(\phi_{1}(x_{t}),\ldots,\phi_{d+k}(x_{t})) - Y_{t} )^{2},
\end{equation*}
the above analysis provides a similar guarantee as Corollary~\ref{cor:main}.
Specifically, the only step requiring modification to accomodate this change is the inequality in \eqref{eqn:empirical-diff-nonpositive},
to account for the fact that $(\hat{g}_{1},\hat{g}_{2},\hat{t}_{\sqrt{m}})$ is not quite the minimizer.
However, for $(\hat{f}_{1},\hat{f}_{2},\hat{t})$ as above, 
denoting by $t_{\sqrt{m}}^{\prime}$ the value in $\{ i \lfloor \sqrt{m}/B \rfloor : i \lfloor \sqrt{m}/B \rfloor \leq m+1 \}$
closest to $\hat{t}$, we have
\begin{equation*}
\hat{R}(\hat{g}_{1},\hat{g}_{2},\hat{t}_{\sqrt{m}})
\leq \hat{R}(\hat{f}_{1},\hat{f}_{2},t_{\sqrt{m}}^{\prime})
\leq \frac{4 B}{\sqrt{m}} + \hat{R}(\hat{f}_{1},\hat{f}_{2},\hat{t}).
%& \sum_{t=1}^{\hat{t}_{\sqrt{m}}-1} ( \hat{g}_{1}(\phi_{1}(x_{t}),\ldots,\phi_{d}(x_{t})) - Y_{t} )^{2} + \sum_{t=\hat{t}_{\sqrt{m}}}^{m} ( \hat{g}_{2}(\phi_{1}(x_{t}),\ldots,\phi_{d+k}(x_{t})) - Y_{t} )^{2}
%\\ & \leq \sum_{t=1}^{t_{\sqrt{m}}^{\prime}-1} ( \hat{f}_{1}(\phi_{1}(x_{t}),\ldots,\phi_{d}(x_{t})) - Y_{t} )^{2} + \sum_{t=t_{\sqrt{m}}^{\prime}}^{m} ( \hat{f}_{2}(\phi_{1}(x_{t}),\ldots,\phi_{d+k}(x_{t})) - Y_{t} )^{2}
%\\ & \leq \sqrt{m} 2 B^{2} + \sum_{t=1}^{\hat{t}-1} ( \hat{f}_{1}(\phi_{1}(x_{t}),\ldots,\phi_{d}(x_{t})) - Y_{t} )^{2} + \sum_{t=\hat{t}}^{m} ( \hat{f}_{2}(\phi_{1}(x_{t}),\ldots,\phi_{d+k}(x_{t})) - Y_{t} )^{2}.
\end{equation*}
Plugging this into the above analysis yields that, with probability at least $1-\delta$,
\begin{align*}
&R^{*}(\hat{g}_{1},\hat{g}_{2},\hat{t}_{\sqrt{m}}) 
\leq  R^{*}(f_{1}^{*},f_{2}^{*},t^{*})\\ &
+ 26 B \sqrt{\frac{2 \ln(2 (m+1)/\delta) + \sum_{j=1}^{2} 3 p_{j} \ln( e m B / p_{j} )}{m}}. 
%& \frac{1}{m} \left( \sum_{t=1}^{\hat{t}_{\sqrt{m}-1}} ( \hat{g}_{1}(\phi_{1}(x_{t}),\ldots,\phi_{d}(x_{t})) - \eta_{t} )^{2} + \sum_{t=\hat{t}_{\sqrt{m}}}^{m} ( \hat{g}_{2}(\phi_{1}(x_{t}),\ldots,\phi_{d+k}(x_{t})) - \eta_{t} )^{2} \right)
%\\ & \leq \frac{1}{m} \left( \sum_{t=1}^{t^{*}-1} ( f_{1}^{*}(\phi_{1}(x_{t}),\ldots,\phi_{d}(x_{t})) - \eta_{t} )^{2} + \sum_{t=t^{*}}^{m} ( f_{2}^{*}(\phi_{1}(x_{t}),\ldots,\phi_{d+k}(x_{t})) - \eta_{t} )^{2} \right)
%\\ & + 22 B \sqrt{\frac{2 \ln(2/\delta) + 3 p_{1} \ln( e m B / p_{1} ) + 3 p_{2} \ln( e m B / p_{2} )}{m}} + \frac{2 B^{2}}{\sqrt{m}}. 
\end{align*}

\paragraph{Remarks on Adapting to Multiple Change Times:}
Rather than allowing only a \emph{single} change time, it is a simple matter to generalize the above
procedure to allow \emph{any number} $K$ of change times.  Specifically, with $K+1$ spaces $\H_{1,K},\ldots,\H_{K+1,K}$,
where $p_{j,K}$ denotes the pseudo-dimension of $\H_{j,K}$,
defining $t_{0} = 0$, $t_{K+1}^{(K)} = m$, 
\begin{equation*}
R^{*}(h_{1},\ldots,h_{K+1},t_{1},\ldots,t_{K}) 
= \frac{1}{m} \sum_{j = 0}^{K} \sum_{t=t_{j}+1}^{t_{j+1}} ( h_{j+1}(x_{t}) - \eta_{t} )^{2},
\end{equation*}
\begin{equation*}
\hat{R}(h_{1},\ldots,h_{K+1},t_{1},\ldots,t_{K}) 
= \frac{1}{m} \sum_{j=0}^{K} \sum_{t=t_{j}+1}^{t_{j+1}} ( h_{j+1}(x_{t}) - Y_{t} )^{2},
\end{equation*}
\begin{align*}
&(h_{1,K}^{*},\ldots,h_{K+1,K}^{*},t_{1,K}^{*},\ldots,t_{K,K}^{*})\\ 
&= \!\!\argmin_{\substack{(h_{1},\ldots,h_{K+1}) \in \times_{j=1}^{K+1} \H_{j,K} \\ 0 \leq t_{1} \leq \cdots \leq t_{K} \leq m}} \!\!R^{*}(h_{1},\ldots,h_{K+1},t_{1},\ldots,t_{K}),
\end{align*}
and
\begin{align*}
&(\hat{h}_{1,K},\ldots,\hat{h}_{K+1,K},\hat{t}_{1,K},\ldots,\hat{t}_{K,K}) \\
&= \!\!\argmin_{\substack{(h_{1},\ldots,h_{K+1}) \in \times_{j=1}^{K+1} \H_{j,K} \\ 0 \leq t_{1} \leq \cdots \leq t_{K} \leq m}}\!\! \hat{R}(h_{1},\ldots,h_{K+1},t_{1},\ldots,t_{K}),
\end{align*}
we have with probability at least $1-\delta$, 
\begin{align*}
 &R^{*}(\hat{h}_{1,K},\ldots,\hat{h}_{K+1,K},\hat{t}_{1,K},\ldots,\hat{t}_{K,K}) \\
 &\leq R^{*}(h_{1,K}^{*},\ldots,h_{K+1,K}^{*},t_{1,K}^{*},\ldots,t_{K,K}^{*}) \\
 &~~~+ 22 B \sqrt{\frac{2 \ln\left(\frac{2 (m+1)^{K}}{\delta}\right) + \sum\limits_{j=1}^{K+1} 3 p_{j,K} \ln( \frac{e m B}{p_{j,K}} )}{m}}.
\end{align*}
The proof follows analogously to that of Theorem~\ref{thm:main}. 

If $K$ is unknown, via the method of \emph{structural risk minimization} \cite{vapnik:82,vapnik:98}, it is still possible to effectively learn. Specifically, for any $K \in \{0,\ldots,m\}$, based on analogous arguments to Theorem~\ref{thm:main}, we have that with probability $1-\delta/(K+2)^{2}$, 
\begin{align*}
& R^{*}(\hat{h}_{1,K},\ldots,\hat{h}_{K+1,K},\hat{t}_{1,K},\ldots,\hat{t}_{K,K}) + \frac{1}{m} \sum_{t=1}^{m} {\rm Var}(Y_{t})\\
 & \leq \hat{R}(\hat{h}_{1,K},\ldots,\hat{h}_{K+1,K},\hat{t}_{1,K},\ldots,\hat{t}_{K,K}) 
%%% adding \epsilon^2 + 4\epsilon B
 \\& + 11 B \sqrt{\frac{2\ln\left(\frac{2 (m+1)^{K} (K+2)^{2}}{\delta}\right) + \sum\limits_{j=1}^{K+1} 3 p_{j,K} \ln(\frac{e m B}{p_{j,K}})}{m}},
\end{align*}
while
\begin{align*}
&\hat{R}(\hat{h}_{1,K},\ldots,\hat{h}_{K+1,K},\hat{t}_{1,K},\ldots,\hat{t}_{K,K}) \\
 &\leq \hat{R}(h_{1,K}^{*},\ldots,h_{K+1,K}^{*},t_{1,K}^{*},\ldots,t_{K,K}^{*}),
\end{align*}
and 
\begin{align*}
&\hat{R}(h_{1,K}^{*},\ldots,h_{K+1,K}^{*},t_{1,K}^{*},\ldots,t_{K,K}^{*})\\
%\\ & + 11 B \sqrt{\frac{2\ln\left(\frac{2 (m+1)^{K} (K+2)^{2}}{\delta}\right) + \sum\limits_{j=1}^{K+1} 3 p_{j,K} \ln(\frac{e m B}{p_{j,K}})}{m}}
 & \leq R^{*}(h_{1,K}^{*},\ldots,h_{K+1,K}^{*},t_{1,K}^{*},\ldots,t_{K,K}^{*}) 
 + \frac{1}{m} \sum_{t=1}^{m} {\rm Var}(Y_{t})
\\ & + 11 B \sqrt{\frac{2\ln\left(\frac{2 (m+1)^{K} (K+2)^{2}}{\delta}\right) + \sum\limits_{j=1}^{K+1} 3 p_{j,K} \ln(\frac{e m B}{p_{j,K}})}{m}}.
\end{align*}
Therefore, choosing
\begin{multline*}
\hat{K} = \!\!\argmin_{K \in \{0,\ldots,m\}}\!\! \hat{R}(\hat{h}_{1,K},\ldots,\hat{h}_{K+1,K},\hat{t}_{1,K},\ldots,\hat{t}_{K,K}) 
\\ + 11 B \sqrt{\frac{2\ln\!\left(\frac{2 (m+1)^{K} (K+2)^{2}}{\delta}\right) + \!\sum\limits_{j=1}^{K+1}\! 3 p_{j,K} \ln(\frac{e m B}{p_{j,K}})}{m}},
\end{multline*}
by a union bound, we have that with probability $1-\delta$, 
\begin{align*}
&R^{*}(\hat{h}_{1,\hat{K}},\ldots,\hat{h}_{\hat{K}+1,\hat{K}},\hat{t}_{1,\hat{K}},\ldots,\hat{t}_{\hat{K},\hat{K}}) \\
 & \leq \min_{K \in \{0,\ldots,m\}} R^{*}(h_{1,K}^{*},\ldots,h_{K+1,K}^{*},t_{1,K}^{*},\ldots,t_{K,K}^{*})
\\ & + 22 B \sqrt{\frac{2\ln\left(\frac{2 (m+1)^{K} (K+2)^{2}}{\delta}\right) + \sum\limits_{j=1}^{K+1} 3 p_{j,K} \ln(\frac{e m B}{p_{j,K}})}{m}}.
\end{align*}
We are thus able to achieve roughly the same guarantee as available above for any fixed $K$.

\begin{figure*}[t]
   \begin{tabular}{c}
      \includegraphics[width=0.55\linewidth]{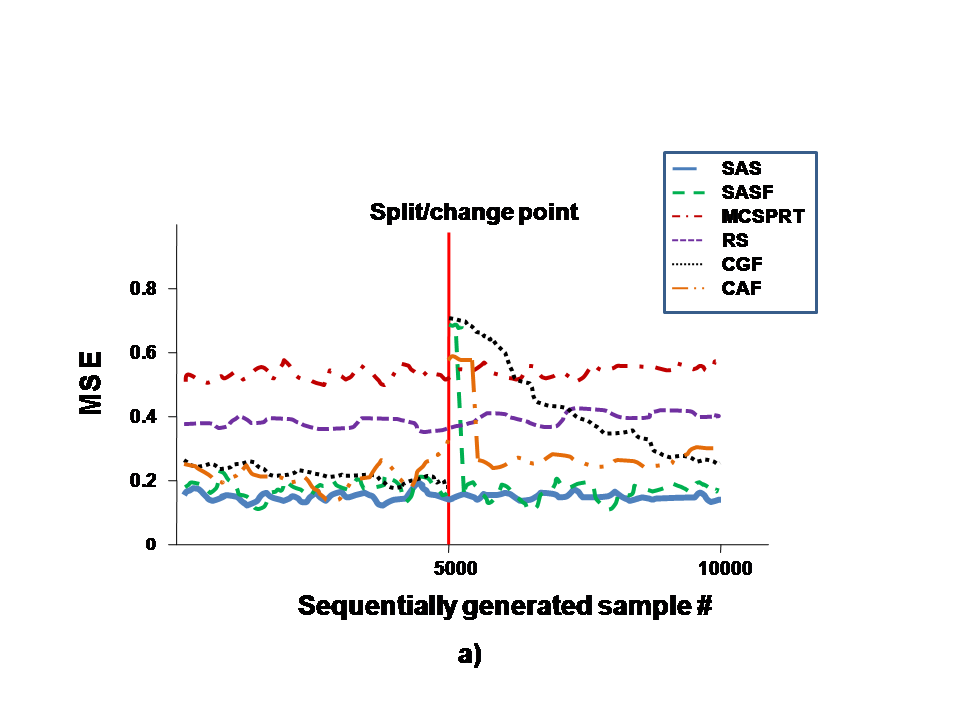}
     \includegraphics[width=0.55\linewidth]{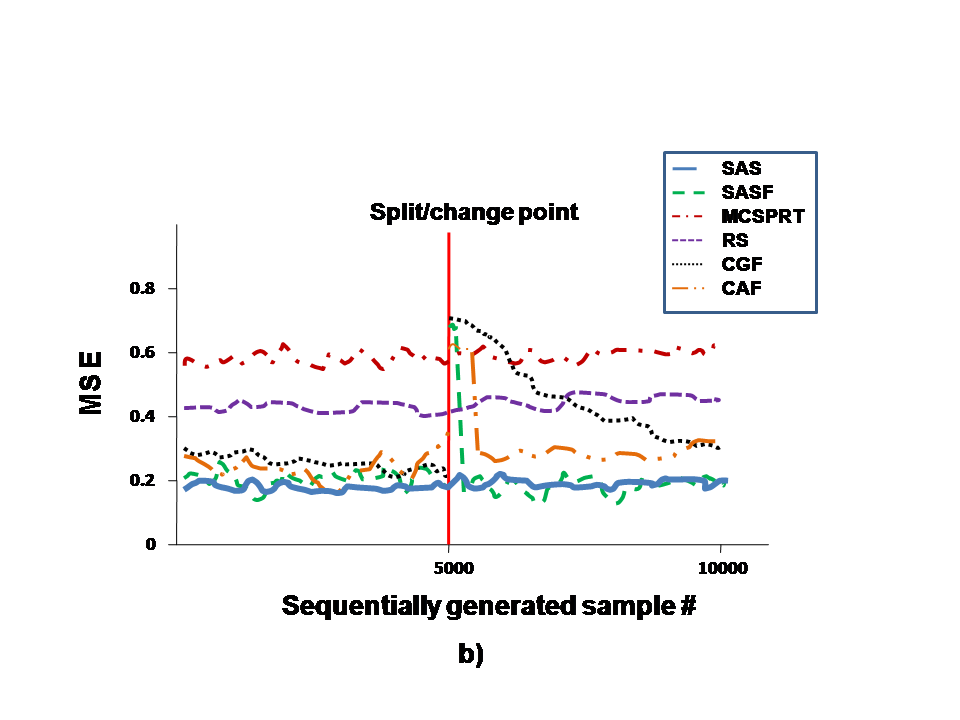}\\
\includegraphics[width=0.55\linewidth]{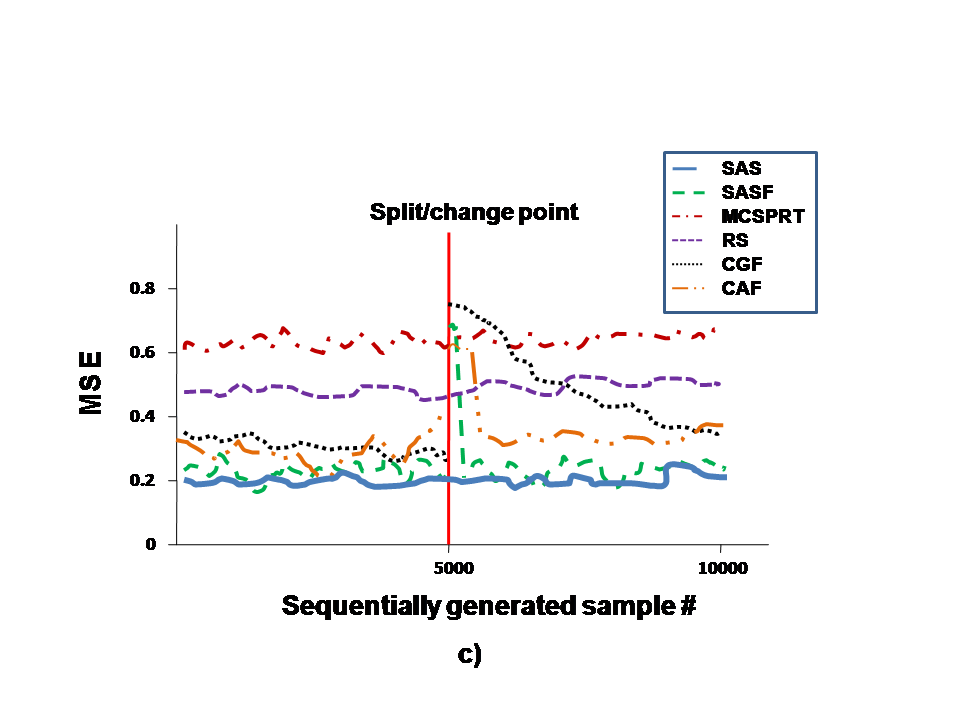}
\includegraphics[width=0.5\linewidth]{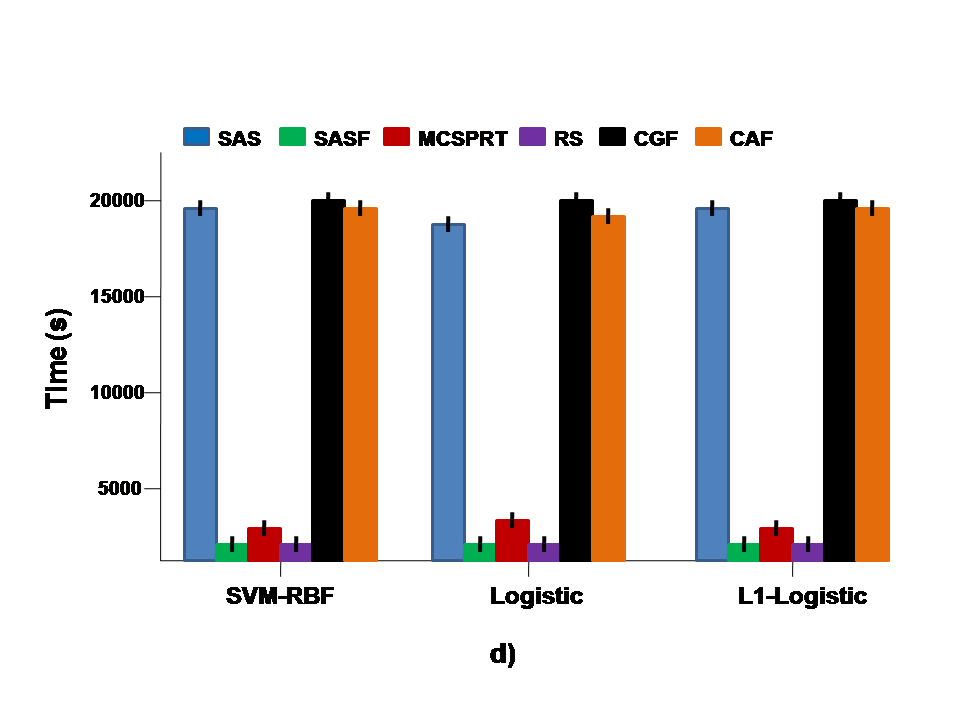}
     \end{tabular}
    \caption{Above we see the performance of the different split point detection approaches with a) SVM-RBF, b) logistic regression and c) L1 reg. logistic regression, as the base learning methods on the synthetic data. The red bold line at 5000 is where the true change has occurred. In d), we see the running time of the different methods.}
\label{syn}
\end{figure*}

\section{Efficient Version (SaSF)}
\label{sec:eff}

We have seen that in our algorithm we learn a model at every time
instant leading to $m$ iterations of model building with the new (or
newly interpreted) features. Moreover, we learn from scratch with the
old and new features at these instances. This can be computationally
intensive as we might have thousands of old features while only a few
new and we would be learning $m$ times over the entire set.

We thus suggest a couple of efficient approximations that could be
done to scale our algorithm. Firstly, based on our remarks in section \ref{sec:theory} on computational speedups, we could
consider only time instances that are $\Theta(\sqrt{m})$ apart and hence learn
only $\Theta(\sqrt{m})$ models as opposed to $m+1$ models. This would potentially
degrade the performance only by a constant factor as confirmed by the bound derived in that subsection. Secondly, we could
use the meta-algorithm suggested in \cite{jmlr2014} to efficiently
update our models with the newly added features. This would lead to a
(learning) time complexity proportional only to the newly added
features. Moreover, their algorithm is optimal for generalized linear
models and more accurate than optimizing the residual for other
models.

The above two strategies should help in making our SaS algorithm significantly
faster, while maintaining accuracy;  we thus refer to the resulting algorithm based on these two speedups\footnote{Ofcourse, the second speedup is only relevant when new features are added.} as SaSF, an acronym for SaS Fast.

\begin{figure*}[t]
   \begin{tabular}{c}
      \includegraphics[width=0.55\linewidth]{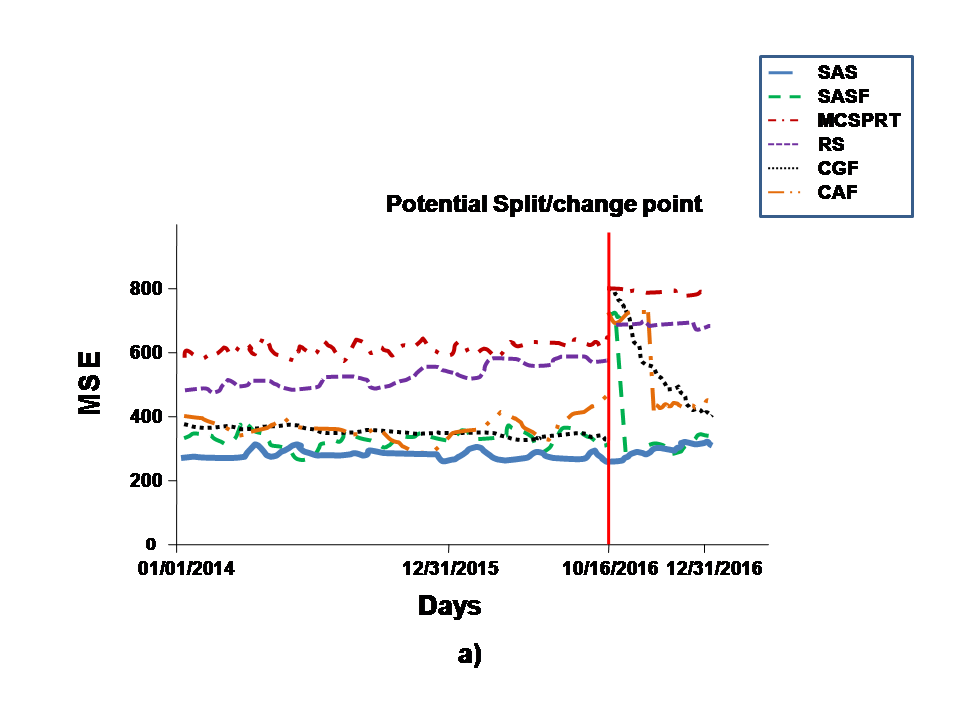}
     \includegraphics[width=0.5\linewidth]{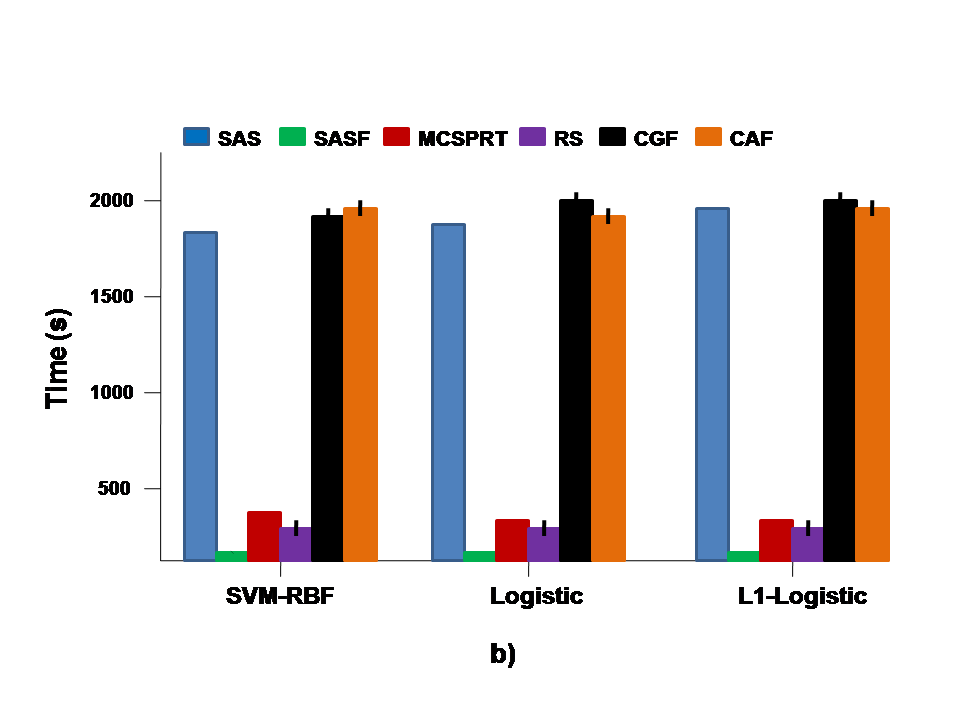}\\
\includegraphics[width=0.55\linewidth]{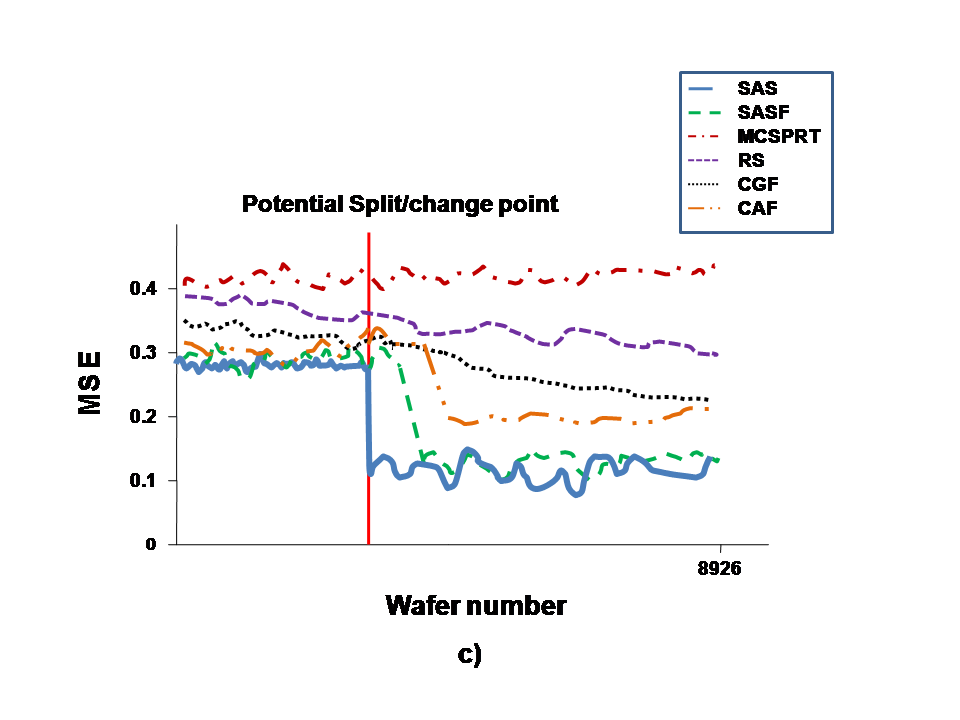}
\includegraphics[width=0.5\linewidth]{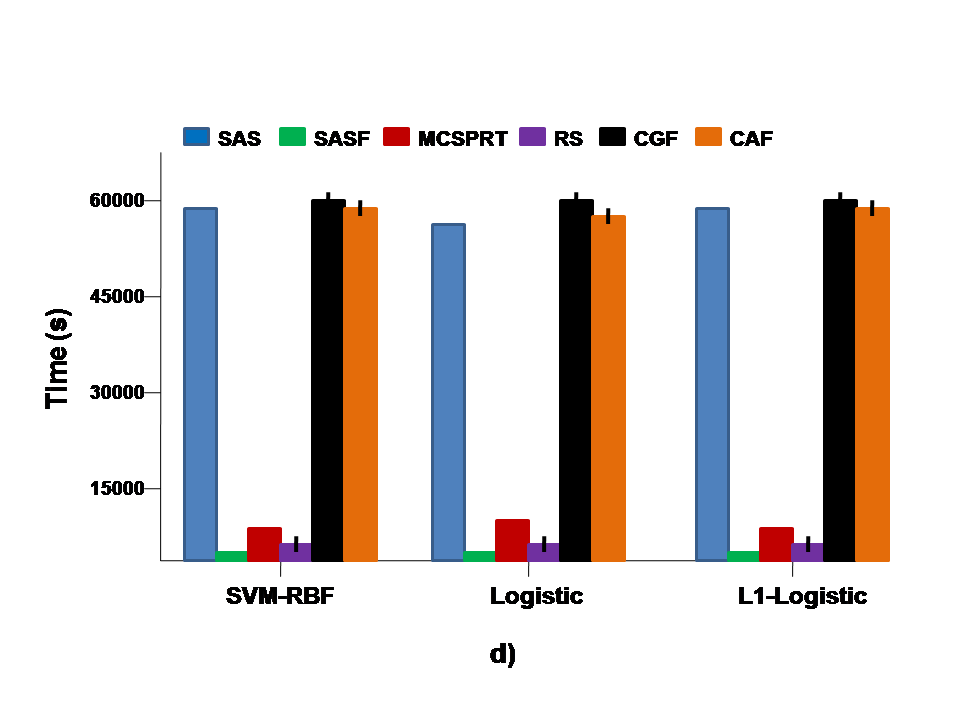}
     \end{tabular}
    \caption{Above we see the performance on the real datasets. a) and c) are the accuracy results on the retail dataset and the manufacturing dataset respectively, with SVM-RBF as the base learning method. Due to space constraints results with logistic and L1-logistic regression are in the supplementary material. b) and d) are the timing results on the retail dataset and manufacturing dataset respectively.}
\label{real}
\end{figure*}

\section{Experiments}
\label{sec:exp}

To support the theory, in this section, we evaluate our approach on synthetic as well as on two real industrial datasets. Since our setting is novel and we do not have direct competitors, we compare against methods from related frameworks described before. In particular, we compare against i) two concept drift methods one based on gradual forgetting (CGF) and the second based on abrupt forgetting (CAF)\cite{cdSurvey}, ii) a multivariate unsupervised change point detection method namely, multivariate cusum based on sequential probability ratio test (MCSPRT) and iii) a randomized strawman which randomly determines the split point (RS). Moreover, to showcase the fact that our approach's performance is not tied to any specific learning algorithm we experiment with three algorithms namely, logistic regression, L1-regularized logistic regression and SVM with an RBF kernel. 

For our approaches, RS and MCSPRT each learning algorithm is trained separately on the data before and after the identified change/split by these methods. If no split point is detected by a particular method, we train the algorithm over the entire dataset. We report the mean squared error (MSE) for each datapoint on the time ordered datasets we consider, based on the prediction of the learned hypothesis that the point corresponds to i.e. either before or after the split. In other words, points before the identified split point are predicted with the hypothesis trained on the data before the split and analogously points following the split point are predicted with hypothesis trained on the data after the split. For RS we average the results over 100 different randomly determined splits.

For the concept drift approaches the algorithms are trained based on a discounted weighting scheme for CGF and windowing for CAF as described in \cite{cdSurvey}. We experimented with 100 different $\lambda$ parameters for CGF and 100 different window sizes for CAF and chose the best.

\subsection{Synthetic Setup}

We generate data using two 1001 dimensional gaussians with mean zero -- corresponding to data before and after the split -- where the last feature is the output variable. Hence, we have 1000 input variables. We set the correlation between all the input variables to be a small value such as 0.2, as it is more realistic than having extremely high correlation or at the other end of the spectrum all of them being independent. Usually in practice many statisticians working on applied problems consider a correlation less than 0.2 to be a proxy for independence. So a value of 0.2 is indicative of weak dependence. In the first gaussian, we set the first 500 features to have very low correlation ($\rho=0.1$) with the output, while the next 500 features have high correlation ($\rho=0.7$) with the output. In the second gaussian the situation is reversed and the first 500 features have high correlation ($\rho=0.7$) with the output while the next 500 features have very low correlation ($\rho=0.1$) with the output.

We generate 100 datasets of size 10000 (i.e. $m=10000$) where the first 5000 points in each dataset are generated by the first gaussian, while the second 5000 are generated by the second gaussian. We generate the points sequentially so that there is an explicit ordering. We also normalize the outputs between zero and one so as to effectively apply logistic regression and its regularized counterpart.

\subsection{Real Data Setup}

The first dataset we consider is from a large retailer. We have 2 years of online customer data from the beginning of 2015 to the end of 2016. This is around 4TB of data containing information of roughly 80 million customers. The goal was to predict daily sales of a line/brand based on derived (aggregated) attributes from customer online visits that day such as the most common geo or zip, number of loyalty customers, average time spent on a visit, average number of pages visited per visit, average number of visits, average brand affinity across the customers that visited, average color affinity and average finish affinity for such customers. Although actual daily sales would be recorded at the end of day, having a good model can inform the business of which factors are important in identifying sales and provide confidence that they are in fact monitoring the right factors. Moreover, if sales drastically change or if factors change it is important to detect and consequently investigate the reasons. We report the results on one of the brands that was high priority for the business.

The second real dataset we consider is a real semi-conductor process of
microprocessor or chip production. In our data, a single datapoint is
a wafer, which is a group of chips, and measurements, which correspond
to input features (temperatures, pressures, etc.), are made on this wafer throughout its
production. The target that is used to evaluate the quality of the wafer,
in this case, is the (normalized) speed of the wafer, which is the median of the
speeds of its chips. We consider two critical stages of the manufacturing process where we have 2287 measurements at the first stage and another 1030 added at the second stage. The total number of wafers is 8926. The goal is to figure out if and when the second stage features start affecting wafer speed. Hence, any virtual metrology models (i.e. regression models) that have been built with the historical data have to be updated from this point onwards.

\subsection{Observations}

\noindent\textit{Synthetic Experiments:} In figures \ref{syn}a, \ref{syn}b and \ref{syn}c we see the perfomance of the different approaches on the synthetic dataset. We observe that our SaS method is by far the best at adapting to the split point as it seamlessly chooses the appropriate hypothesis for points before and after the change. SaSF is not as adept, however it too quickly switches to the right hypothesis after the change. CAF seems to be the best amongst the competitors, however it takes significantly longer to adapt with its performance being consistently slightly worse after the change. CGF slowly improves after the change point but is significantly worse than both SaS and SaSF. Random and MCSPRT do not improve with time. However random is better than MCSPRT as with random at least we have two hypothesis that are used to predict, while with MCSPRT since there is no change in the input distribution, no change point is identified and we hence learn a single regressor over the whole dataset. In figure \ref{syn}d, we see that SaSF is amongst the fastest methods and significantly faster than SaS. Hence, although SaS might be the best in performance, SaSF is definitely preferable if efficiency is a concern.

\noindent\textit{Real Data Experiments:} In figures \ref{real}a and \ref{real}b we see the performance and timing results on the retail dataset. Figure \ref{real}a shows the results with SVM-RBF, however results with the other two methods that are in the supplementary material are qualitatively similar. We observe that around $16^{th}$ of October our method SaS identifies a change point. At this juncture the other methods fail to detect the change point and keep using the same hypothesis resulting in bad predictions. SaSF although not as good as SaS adapts faster than the other competitors. Moreover, it is significantly faster than all the other methods as is seen in figure \ref{real}b. Post analysis of the change point by the domain experts resulted in the finding that around that time certain unflattering news stories were released which tarnished the brand image resulting in this change. This was an independent justification for the identified change point and a potential confirmation that it was not just noise.

In figures \ref{real}c and \ref{real}d we see the performance and timing results on the manufacturing dataset. In figure \ref{real}c, we observe that after our identified change point the performance suddenly improves. If we investigate the regressor after the change point we see that the new features start having high correlation with the target which results in this dramatic improvement. As before we see that SaS is the fastest in detecting, while SaSF although slower is still much better than the other methods. In terms of time, we observe in figure \ref{real}d, that SaSF has the biggest gain when compared with the previous experiments. The reason being that we can use both the speedups i.e. only $\sqrt{m}$ models to learn and the efficient feature updating strategy.

\section{Discussion}
Based on the theory and experiments we find that our method 
of determining the empirically optimal choice of split point 
leads to the best models in these applications.

In the future it would be interesting to experiment with multiple
change points, where the number is not a priori known. The
method based on structural risk minimization, discussed in Section~\ref{sec:theory}, should be effective for this.  This method appears to present computational
challenges when searching over many change points.  However, 
it may be possible to reduce the time complexity through dynamic
programming.

\section{Supplement: Results for Regularized Learners with Large Function Classes}
\label{sec:regularization}

In this supplement, we prove an extension of Theorem 1 from the main article,
to allow for regularization terms in the optimization defining the learning algorithm.
This extends the theory to cover methods such as $L_1$ regularized logistic regression
and Support Vector Regression.  The derived results are analogous to those presented
for empirical risk minimization over classes with finite pseudo-dimension in the main 
article, but are included here for completeness.

The setting is similar to that in Theorem 1, except that now $p_{i}$ may be infinite, and we will specify $\hat{h}_{i}$
using a penalized empirical risk minimization setup.
Specifically, we now denote by $\pen_{i}(q,t)$ a value in $[0,\infty)$, for any given $q \in [0,\infty)$.
Then, for each $i \in \{1,2\}$ and $t \in \{1,\ldots,m+1\}$, we let $\H_{i,t}^{(q)}$ be a family of sets (indexed by $q \in [0,\infty)$)
that is nondecreasing in $q$: that is, for $q < q^{\prime}$, $\H_{i,t}^{(q)} \subseteq \H_{i,t}^{(q^{\prime})}$.
Then we let 
\begin{equation*}
(\hat{h}_{1},\hat{h}_{2},\hat{t},\hat{q}_{1},\hat{q}_{2}) =\!\!\!\!\!\! \argmin_{\substack{(h_{1},h_{2},t_{0},q_{1},q_{2}) : \\ h_{1} \in \H_{1,t_{0}}^{(q_{1})}, h_{2} \in \H_{2,t_{0}}^{(q_{2})}, \\ t_{0} \in \{1,\ldots,m+1\}, q_{1},q_{2} \in [0,\infty)}} \!\!\!\!\!\! m \hat{R}(h_{1},h_{2},t_{0}) + \pen_{1}(q_{1},t_{0}) + \pen_{2}(q_{2},t_{0}).
\end{equation*}
Also denote by
\begin{equation*}
(h_{1}^{*},h_{2}^{*},t^{*},q_{1}^{*},q_{2}^{*}) =\!\!\!\!\!\! \argmin_{\substack{(h_{1},h_{2},t_{0},q_{1},q_{2}) : \\ h_{1} \in \H_{1,t_{0}}^{(q_{1})}, h_{2} \in \H_{2,t_{0}}^{(q_{2})}, \\ t_{0} \in \{1,\ldots,m+1\}, q_{1},q_{2} \in [0,\infty)}} \!\!\!\!\!\! m R^{*}(h_{1},h_{2},t_{0}) + \pen_{1}(q_{1},t_{0}) + \pen_{2}(q_{2},t_{0}).
\end{equation*}
%Now let $\H_{i}^{(q)} = \bigcup_{t = 1}^{m+1} \H_{i,t}^{(q)}$,
%and as in the proof of Theorem 1, let $\H_{i,\epsilon}^{(q)}$ denote a minimal $\epsilon$-cover of $\H_{i}^{(q)}$ with respect to the pseudo-metric $\rho_{m}$.
%Now let $\H_{i}^{(q)} = \bigcup_{t = 1}^{m+1} \H_{i,t}^{(q)}$,
As in the proof of Theorem 1, let $\H_{i,t,\epsilon}^{(q)}$ denote a minimal $\epsilon$-cover of $\H_{i,t}^{(q)}$ with respect to the pseudo-metric $\rho_{m}$.
Then let $N_{i}(\epsilon,q)$ be any value satisfying $N_{i}(\epsilon,q) \geq \max_{t \in \{1,\ldots,m+1\}} |\H_{i,t,\epsilon}^{(q)}|$.
%Then let $N_{i}(\epsilon,q) = |\H_{i,\epsilon}^{(q)}|$.
Also let $\epsilon_{i}(q)$ denote any value in $(0,B]$ such that
\begin{equation*}
%\epsilon_{i}(q) \leq B\sqrt{\frac{\ln(|\H_{i,t,\epsilon_{i}(q)}^{(q)}|)}{m}},
\epsilon_{i}(q) \leq B\sqrt{\frac{\ln(N_{i}(\epsilon_{i}(q),q))}{m}}.
\end{equation*}
For brevity, also denote $N_{i}(q) = N_{i}(\epsilon_{i}(q),q)$.

We have the following result.

\paragraph{Theorem A1.} With probability at least $1-\delta$, 
\begin{align*}
& R^{*}(\hat{h}_{1},\hat{h}_{2},\hat{t}) \leq R^{*}(h_{1}^{*},h_{2}^{*},t^{*}) + \frac{\pen_{1}(q_{1}^{*},t^{*}) + \pen_{2}(q_{2}^{*},t^{*})}{m} 
\\ & + 22 B^{2} \sqrt{\frac{2 \ln(\frac{2(m+1)}{\delta}) \!+\! \sum_{i=1}^{2} \sum_{q \in \{\hat{q}_{i},q_{i}^{*}\}} \!\ln\!\left( (q\!+\!3)^{2} N_{i}(\lceil q \rceil) \right)}{m}}
%|\H_{i,\hat{t},\epsilon_{i,\hat{t}}(\lceil \hat{q}_{i} \rceil)}^{(\lceil \hat{q}_{i} \rceil)}|\right) \!+\! \ln\!\left( (q_{i}^{*}\!+\!3)^{2}|\H_{i,t^{*},\epsilon_{i,t^{*}}(\lceil q_{i}^{*} \rceil)}^{(\lceil q_{i}^{*} \rceil)}| \right)}{m}}.
\end{align*}

\begin{proof}
We begin similarly to the proof of Theorem 1.
Let $\tilde{Y}_{1},\ldots,\tilde{Y}_{m}$ be equal in distribution to $Y_{1},\ldots,Y_{m}$
but independent of $Y_{1},\ldots,Y_{m}$, and define $\tilde{R}(h_{1},h_{2},t_{0})$ as in 
the proof of Theorem 1.
Again, for any $h_{1}$, $h_{2}$, and any $t \in \{1,\ldots,m+1\}$, 
for any $\conf_{q_{1},q_{2}}^{\prime} \in (0,1)$, 
Hoeffding's inequality implies that with probability at least $1-\delta_{q_{1},q_{2}}^{\prime}$, 
\begin{equation}
\label{eqn:tilde-concentration}
\left| \hat{R}(h_{1},h_{2},t) - \E[\tilde{R}(h_{1},h_{2},t)] \right| \leq \sqrt{\frac{2B^{2}\ln(2/\delta_{q_{1},q_{2}}^{\prime})}{m}}.
\end{equation}
For any given $q_{1}, q_{2} \in [0,\infty)$, 
by the union bound, \eqref{eqn:tilde-concentration} holds simultaneously for every choice of 
$t \in \{1,\ldots,m+1\}$, $h_{1} \in \H_{1,t,\epsilon_{1}(q_{1})}^{(q_{1})}$, and $h_{2} \in \H_{2,t,\epsilon_{2}(q_{2})}^{(q_{2})}$, 
with probability at least $1 - \delta_{q_{1},q_{2}}^{\prime} N_{1}(q_{1}) N_{2}(q_{2}) (m+1)$.
Thus, for any $\delta_{q_{1},q_{2}} \in (0,1)$, 
taking $\delta_{q_{1},q_{2}}^{\prime} = \frac{\delta_{q_{1},q_{2}}}{N_{1}(q_{1}) N_{2}(q_{2}) (m+1)}$, 
we have that with probability at least $1-\delta_{q_{1},q_{2}}$, 
$\forall t \in \{1,\ldots,m+1\}$,
$\forall h_{1} \in \H_{1,t,\epsilon_{1}(q_{1})}^{(q_{1})}$, $\forall h_{2} \in \H_{2,t,\epsilon_{2}(q_{2})}^{(q_{2})}$,
\begin{equation*}
\left| \hat{R}(h_{1},h_{2},t) - \E[\tilde{R}(h_{1},h_{2},t)] \right| \leq \sqrt{\frac{2B^{2}}{m}\ln\left(\frac{2 (m+1) N_{1}(q_{1}) N_{2}(q_{2}) }{\delta_{q_{1},q_{2}}}\right)}.
\end{equation*}
Furthermore, by the union bound, this fact holds simultaneously for all $q_{1},q_{2} \in \nats \cup \{0\}$ with probabilty at least 
$1 - \sum_{q_{1}, q_{2} \in \nats\cup\{0\}} \delta_{q_{1},q_{2}}$.  In particular, taking $\delta_{q_{1},q_{2}} = \frac{\delta}{(q_{1}+2)^{2}(q_{2}+2)^{2}}$,
we have $\sum_{q_{1},q_{2} \in \nats\cup\{0\}} \delta_{q_{1},q_{2}} < \delta$.
Thus, with probability at least $1-\delta$, 
$\forall q_{1},q_{2} \in \nats \cup \{0\}$, 
$\forall t \in \{1,\ldots,m+1\}$,
$\forall h_{1} \in \H_{1,t,\epsilon_{1}(q_{1})}^{(q_{1})}$, $\forall h_{2} \in \H_{2,t,\epsilon_{2}(q_{2})}^{(q_{2})}$,
\begin{multline*}
\left| \hat{R}(h_{1},h_{2},t) - \E[\tilde{R}(h_{1},h_{2},t)] \right| 
\\ \leq \sqrt{\frac{2B^{2}}{m}\ln\left(\frac{2 (m+1) (q_{1}+2)^{2} (q_{2}+2)^{2} N_{1}(q_{1}) N_{2}(q_{2})}{\delta}\right)}.
\end{multline*}
To extend this to allow general values of $q_{1},q_{2} \in [0,\infty)$, we simply round up to the next integer:
that is, on the above event of probability at least $1-\delta$, 
for any $q_{1},q_{2} \in [0,\infty)$, 
$\forall t \in \{1,\ldots,m+1\}$,
$\forall h_{1} \in \H_{1,t,\epsilon_{1}(\lceil q_{1} \rceil)}^{(\lceil q_{1} \rceil)}$, $h_{2} \in \H_{2,t,\epsilon_{2}(\lceil q_{2} \rceil)}^{(\lceil q_{2} \rceil)}$,
\begin{multline*}
%\label{eqn:uniform-tilde-concentration}
\left| \hat{R}(h_{1},h_{2},t) - \E[\tilde{R}(h_{1},h_{2},t)] \right| 
\\ \leq \sqrt{\frac{2B^{2}}{m}\ln\left(\frac{2 (m+1) (q_{1}+3)^{2} (q_{2}+3)^{2} N_{1}(\lceil q_{1} \rceil) N_{2}(\lceil q_{2} \rceil)}{\delta}\right)}.
\end{multline*}

Now, for any $q_{1},q_{2} \in [0,\infty)$, define 
$h_{i,q}^{*} = \argmin_{h \in \H_{i,t^{*},\epsilon_{i}( \lceil q \rceil)}^{(\lceil q \rceil)}} \rho_{m}(h,h_{i}^{*})$
and $\hat{h}_{i,q} = \argmin_{h \in \H_{i,\hat{t},\epsilon_{i}( \lceil q \rceil)}^{(\lceil q \rceil)}} \rho_{m}(h,\hat{h}_{i})$, for $i \in \{1,2\}$.
Denote $\hat{y}_{t} = \hat{h}_{1}(x_{t})$ and $\hat{y}_{t,q_{1},q_{2}} = \hat{h}_{1,q_{1}}(x_{t})$ for each $t \leq \hat{t}-1$,
and $\hat{y}_{t} = \hat{h}_{2}(x_{t})$ and $\hat{y}_{t,q_{1},q_{2}} = \hat{h}_{2,q_{2}}(x_{t})$ for each $t \geq \hat{t}$.
Also let $\hat{\epsilon} = \max_{i \in \{1,2\}} \epsilon_{i}( \lceil \hat{q}_{i} \rceil )$.
Similarly, denote 
$y_{t}^{*} = h_{1}^{*}(x_{t})$ and $y_{t,q_{1},q_{2}}^{*} = h_{1,q_{1}}^{*}(x_{t})$ for each $t \leq t^{*}-1$,
and $y_{t}^{*} = h_{2}^{*}(x_{t})$ and $y_{t,q_{1},q_{2}}^{*} = h_{2,q_{2}}^{*}(x_{t})$ for each $t \geq t^{*}$,
and let $\epsilon^{*} = \max_{i \in \{1,2\}} \epsilon_{i}( \lceil q_{i}^{*} \rceil )$.
Then by straightforward calculations, we have
\begin{align*}
& R^{*}(\hat{h}_{1},\hat{h}_{2},\hat{t}) - R^{*}(h_{1}^{*},h_{2}^{*},t^{*})
= \E\left[ \tilde{R}(\hat{h}_{1},\hat{h}_{2},\hat{t}) | \hat{h}_{1} \hat{h}_{2}, \hat{t}, \hat{q}_{1}, \hat{q}_{2} \right] - \E\left[ \tilde{R}(h_{1}^{*},h_{2}^{*},t^{*}) \right]
\\ & \leq \frac{1}{m}\sum_{t=1}^{m} \E\left[ (\hat{y}_{t,\hat{q}_{1},\hat{q}_{2}} - \tilde{Y}_{t})^{2} + \hat{\epsilon}^{2} + 2 \hat{\epsilon} | \hat{y}_{t,\hat{q}_{1},\hat{q}_{2}} - \tilde{Y}_{t} | \middle| \hat{h}_{1},\hat{h}_{2},\hat{t},\hat{q}_{1},\hat{q}_{2} \right]
\\ & {\hskip 6mm}- \frac{1}{m} \sum_{t=1}^{m} \E\left[ (y_{t,q_{1}^{*},q_{2}^{*}}^{*} - \tilde{Y}_{t})^{2} - (\epsilon^{*})^{2} - 2 \epsilon^{*} | y_{t,q_{1}^{*},q_{2}^{*}} - \tilde{Y}_{t} | \right]
\\ & \leq \hat{\epsilon}^{2} + (\epsilon^{*})^{2} + 4 ( \hat{\epsilon} + \epsilon^{*} ) B 
\\ & {\hskip 6mm}+ \E\left[ \tilde{R}(\hat{h}_{1,\hat{q}_{1}},\hat{h}_{2,\hat{q}_{2}},\hat{t}) \middle| \hat{h}_{1}, \hat{h}_{2}, \hat{t}, \hat{q}_{1}, \hat{q}_{2} \right] - \E\left[ \tilde{R}(h_{1,q_{1}^{*}}^{*}, h_{2,q_{2}^{*}}^{*}, t^{*}) \right].
\end{align*}
Thus, on the above event of probability $1-\delta$, 
\begin{align*}
& R^{*}(\hat{h}_{1},\hat{h}_{2},\hat{t}) - R^{*}(h_{1}^{*},h_{2}^{*},t^{*})
\leq \hat{\epsilon}^{2} + (\epsilon^{*})^{2} + 4 ( \hat{\epsilon} + \epsilon^{*} ) B 
\\ & + \sqrt{\frac{2B^{2}}{m}\ln\left(\frac{2 (m+1) (\hat{q}_{1}+3)^{2} (\hat{q}_{2}+3)^{2} N_{1}(\lceil \hat{q}_{1} \rceil) N_{2}(\lceil \hat{q}_{2} \rceil)}{\delta}\right)}
\\ & + \sqrt{\frac{2B^{2}}{m}\ln\left(\frac{2 (m+1) (q_{1}^{*}+3)^{2} (q_{2}^{*}+3)^{2} N_{1}(\lceil q_{1}^{*} \rceil) N_{2}(\lceil q_{2}^{*} \rceil)}{\delta}\right)}
\\ & + \hat{R}(\hat{h}_{1,\hat{q}_{1}},\hat{h}_{2,\hat{q}_{2}},\hat{t}) - \hat{R}(h_{1,q_{1}^{*}}^{*},h_{2,q_{2}^{*}}^{*},t^{*}).
\end{align*}
Then note that
\begin{align*}
& \hat{R}(\hat{h}_{1,\hat{q}_{1}},\hat{h}_{2,\hat{q}_{2}},\hat{t}) - \hat{R}(h_{1,q_{1}^{*}}^{*},h_{2,q_{2}^{*}}^{*},t^{*})
\\ & \leq \frac{1}{m} \sum_{t=1}^{m} \left( (\hat{y}_{t} - Y_{t})^{2} + \hat{\epsilon}^{2} + 2 \hat{\epsilon} |\hat{y}_{t} - Y_{t}| \right)
\\ & {\hskip 12mm}- \frac{1}{m} \sum_{t=1}^{m} \left( (y_{t}^{*} - Y_{t})^{2} - (\epsilon^{*})^{2} - 2 \epsilon^{*} |y_{t}^{*} - Y_{t}| \right)
\\ & \leq \hat{\epsilon}^{2} + (\epsilon^{*})^{2} + 4 B ( \hat{\epsilon} + \epsilon^{*} ) + \hat{R}(\hat{h}_{1},\hat{h}_{2},\hat{t}) - \hat{R}(h_{1}^{*}, h_{2}^{*}, t^{*})
\\ & \leq \hat{\epsilon}^{2} + (\epsilon^{*})^{2} + 4 B ( \hat{\epsilon} + \epsilon^{*} ) + \hat{R}(\hat{h}_{1},\hat{h}_{2},\hat{t}) 
\\ & {\hskip 22mm}+ \frac{\pen_{1}(\hat{q}_{1},\hat{t}) + \pen_{2}(\hat{q}_{2},\hat{t})}{m} - \hat{R}(h_{1}^{*}, h_{2}^{*}, t^{*})
\\ & \leq \hat{\epsilon}^{2} + (\epsilon^{*})^{2} + 4 B ( \hat{\epsilon} + \epsilon^{*} ) + \hat{R}(h_{1}^{*},h_{2}^{*},t^{*}) 
\\ & {\hskip 22mm}+ \frac{\pen_{1}(q_{1}^{*},t^{*}) + \pen_{2}(q_{2}^{*},t^{*})}{m} - \hat{R}(h_{1}^{*}, h_{2}^{*}, t^{*})
\\ & = \hat{\epsilon}^{2} + (\epsilon^{*})^{2} + 4 B ( \hat{\epsilon} + \epsilon^{*} ) + \frac{\pen_{1}(q_{1}^{*},t^{*}) + \pen_{2}(q_{2}^{*},t^{*})}{m}.
\end{align*}
Altogether, we have that with probability at least $1-\delta$, 
\begin{align*}
& R^{*}(\hat{h}_{1},\hat{h}_{2},\hat{t}) - R^{*}(h_{1}^{*},h_{2}^{*},t^{*})
\\ & \leq 2\hat{\epsilon}^{2} + 2(\epsilon^{*})^{2} + 8 ( \hat{\epsilon} + \epsilon^{*} ) B + \frac{\pen_{1}(q_{1}^{*},t^{*}) + \pen_{2}(q_{2}^{*},t^{*})}{m}
\\ & {\hskip 6mm}+ \sqrt{\frac{2B^{2}}{m}\ln\left(\frac{2 (m+1) (\hat{q}_{1}+3)^{2} (\hat{q}_{2}+3)^{2} N_{1}(\lceil \hat{q}_{1} \rceil) N_{2}(\lceil \hat{q}_{2} \rceil)}{\delta}\right)}
\\ & {\hskip 6mm}+ \sqrt{\frac{2B^{2}}{m}\ln\left(\frac{2 (m+1) (q_{1}^{*}+3)^{2} (q_{2}^{*}+3)^{2} N_{1}(\lceil q_{1}^{*} \rceil) N_{2}(\lceil q_{2}^{*} \rceil)}{\delta}\right)}
\\ & \leq 11 B^{2}\sqrt{\frac{2}{m}\ln\left(\frac{2 (m+1) (\hat{q}_{1}+3)^{2} (\hat{q}_{2}+3)^{2} N_{1}(\lceil \hat{q}_{1} \rceil) N_{2}(\lceil \hat{q}_{2} \rceil) }{\delta}\right)}
\\ & {\hskip 6mm}+ 11 B^{2} \sqrt{\frac{2}{m}\ln\left(\frac{2 (m+1) (q_{1}^{*}+3)^{2} (q_{2}^{*}+3)^{2} N_{1}(\lceil q_{1}^{*} \rceil) N_{2}(\lceil q_{2}^{*} \rceil)}{\delta}\right)}
\\ & {\hskip 6mm}+ \frac{\pen_{1}(q_{1}^{*},t^{*}) + \pen_{2}(q_{2}^{*},t^{*})}{m}
\\ & \leq 22 B^{2}\sqrt{\frac{2}{m}\ln\left(\frac{2(m+1)}{\delta} \prod_{i \in \{1,2\}} \prod_{q \in \{\hat{q}_{i},q_{i}^{*}\}} (q+3)^{2} N_{i}(\lceil q \rceil)\right)}
\\ & {\hskip 6mm}+ \frac{\pen_{1}(q_{1}^{*},t^{*}) + \pen_{2}(q_{2}^{*},t^{*})}{m}.
\end{align*}
\end{proof}

As an example application of Theorem A1, consider the case of regularized kernel regression, with kernel functions $K_{1}$, $K_{2}$ (e.g., radial basis functions),
along with feature functions $\phi_{1},\ldots,\phi_{d+k}$ as in Section 3 of the main article.  For brevity, denote by $\phi_{1:d}(x) = (\phi_{1}(x),\ldots,\phi_{d}(x))$ and $\phi_{1:(d+k)}(x) = (\phi_{1}(x),\ldots,\phi_{d+k}(x))$.
%In this case, functions are parametrized by a vector $\alpha_{i:j} = (\alpha_{i},\ldots,\alpha_{j}) \in \reals^{j-i+1}$, where $h_{\alpha_{i:j}}(x) = \sum_{t=i}^{j} \alpha_{t} K_{1}(\phi_{1:x_{t},x)$.
In this case, fix any $\lambda_{1},\lambda_{2} > 0$, and consider choosing 
\begin{equation*}
\hat{h}_{1}(x) = \sum_{t=1}^{\hat{t}-1} \hat{\alpha}_{t} K_{1}( \phi_{1:d}(x_{t}), \phi_{1:d}(x) )
\end{equation*}
and 
\begin{equation*}
\hat{h}_{2}(x) = \sum_{t=\hat{t}}^{m} \hat{\alpha}_{t} K_{2}( \phi_{1:(d+k)}(x_{t}), \phi_{1:(d+k)}(x) ),
\end{equation*}
with parameters $\hat{\alpha}_{1},\ldots,\hat{\alpha}_{m}$
chosen so that $\hat{\alpha}_{1},\ldots,\hat{\alpha}_{\hat{t}-1}$ minimizes the expression
\begin{equation*}
\sum_{t=1}^{\hat{t}-1} \left( Y_{t} - \sum_{i=1}^{\hat{t}-1} \hat{\alpha}_{i} K_{1}(\phi_{1:d}(x_{i}),\phi_{1:d}(x_{t})) \right)^{2} \!\!+ \lambda_{1} \sum_{i=1}^{\hat{t}-1} \sum_{j=1}^{\hat{t}-1} \hat{\alpha}_{i} \hat{\alpha}_{j} K_{1}(\phi_{1:d}(x_{i}),\phi_{1:d}(x_{j})),
\end{equation*}
and $\hat{\alpha}_{\hat{t}},\ldots,\hat{\alpha}_{m}$ minimizes the expression
\begin{multline*}
\sum_{t=\hat{t}}^{m} \left( Y_{t} - \sum_{i=\hat{t}}^{m} \hat{\alpha}_{i} K_{2}(\phi_{1:(d+k)}(x_{i}),\phi_{1:(d+k)}(x_{t})) \right)^{2} 
\\ + \lambda_{2} \sum_{i=\hat{t}}^{m} \sum_{j=\hat{t}}^{m} \hat{\alpha}_{i} \hat{\alpha}_{j} K_{2}(\phi_{1:(d+k)}(x_{i}),\phi_{1:(d+k)}(x_{j})),
\end{multline*}
and where $\hat{t}$ is chosen to minimize the sum of these two expressions obtained at these minimizing $\hat{\alpha}_{1},\ldots,\hat{\alpha}_{m}$ values.
This corresponds to choosing $(\hat{h}_{1},\hat{h}_{2},\hat{t})$ as above, with 
$\H_{1,t}^{(q)}$ defined as the set of functions 
$x \mapsto \sum_{i=1}^{t-1} \alpha_{i} K_{1}(\phi_{1:d}(x_{i}),\phi_{1:d}(x))$ with $\sum_{i=1}^{t-1} \sum_{j=1}^{t-1} \alpha_{i} \alpha_{j} K_{1}(\phi_{1:d}(x_{i}),\phi_{1:d}(x_{j})) \leq q$,
and similarly with $\H_{2,t}^{(q)}$ defined as the set of functions $x \mapsto \sum_{i=t}^{m} \alpha_{i} K_{2}(\phi_{1:(d+k)}(x_{i}),\phi_{1:(d+k)}(x))$ with $\sum_{i=t}^{m} \sum_{j=t}^{m} \alpha_{i} \alpha_{j} K_{2}(\phi_{1:(d+k)}(x_{i}),\phi_{1:(d+k)}(x_{j})) \leq q$,
and with $\pen_{1}(q,t) = \lambda_{1} q$ and $\pen_{2}(q,t) = \lambda_{2} q$.

In this example, the covering numbers $|\H_{i,t,\epsilon}^{(q)}|$ will generally depend on the specific kernel function $K_{i}$,
and have been the subject of much study.  As one concrete example, consider unit-bandwidth Gaussian kernels:
$K_{1}(u,v) = \exp\{ - \|u - v\|^{2} \}$, for $u,v \in \reals^{d}$, and $K_{2}(u,v) = \exp\{ - \|u-v\|^{2} \}$, for $u,v \in \reals^{d+k}$.
In this case, denoting $n_{i} = d + k \ind[i=2]$,
\cite{zhou:02} argues that for $0 < \epsilon \leq q \exp\{ - 90 n_{i}^{2} - 11 n_{i} - 3 \}$, we have 
\begin{equation*}
\ln(|\H_{i,t,\epsilon}^{(q)}|) \leq 4^{n_{i}} (6 n_{i} + 2)( \ln(q/\epsilon) )^{n_{i}+1},
\end{equation*}
for $i \in \{1,2\}$.
Since the right hand side does not depend on $t$, let us take $N_{i}(\epsilon,q)$ equal to this value (for $\epsilon$ in the range specified above).
Thus, for instance, we may take 
\begin{equation*}
\epsilon_{i}(q) = \min\left\{ q \exp\{ - 90 n_{i}^{2} - 11 n_{i} - 3 \}, \sqrt{\frac{1}{m}} \right\}
\end{equation*}
to satisfy the constraint
\begin{equation*}
\epsilon_{i}(q) \leq \min\left\{ q \exp\{ - 90 n_{i}^{2} - 11 n_{i} - 3 \},  B 2^{n_{i}} ( \ln(q/\epsilon_{i}(q)) )^{\frac{n_{i}+1}{2}} \sqrt{ \frac{6 n_{i} + 2}{m} } \right\}.
\end{equation*}
Now, denoting by $\alpha_{1}^{*},\ldots,\alpha_{m}^{*}$ the values such that $h_{1}^{*}(x) = \sum_{i=1}^{t^{*}-1} \alpha_{i}^{*} K_{1}(x_{i},x)$ and $h_{2}^{*}(x) = \sum_{i=t^{*}}^{m} \alpha_{i}^{*} K_{2}(x_{i},x)$,
and denoting 
\begin{align*}
A_{1}^{*} & = \sum_{i=1}^{t^{*}-1} \sum_{j=1}^{t^{*}-1} \alpha_{i}^{*} \alpha_{j}^{*} K_{1}(\phi_{1:d}(x_{i}),\phi_{1:d}(x_{j})), 
\\ A_{2}^{*} & = \sum_{i=t^{*}}^{m} \sum_{j=t^{*}}^{m} \alpha_{i}^{*} \alpha_{j}^{*} K_{2}(\phi_{1:(d+k)}(x_{i}),\phi_{1:(d+k)}(x_{j})),
\\ \hat{A}_{1} & = \sum_{i=1}^{\hat{t}-1} \sum_{j=1}^{\hat{t}-1} \hat{\alpha}_{i} \hat{\alpha}_{j} K_{1}(\phi_{1:d}(x_{i}),\phi_{1:d}(x_{j})),
\\ \text{and } \hat{A}_{2} & = \sum_{i=\hat{t}}^{m} \sum_{j=\hat{t}}^{m} \hat{\alpha}_{i} \hat{\alpha}_{j} K_{2}(\phi_{1:(d+k)}(x_{i}),\phi_{1:(d+k)}(x_{j})),
\end{align*}
the above theorem guarantees that, for some finite numerical constants $c_{1},c_{2}$, with probability at least $1-\delta$, 
\begin{multline*}
R^{*}(\hat{h}_{1},\hat{h}_{2},\hat{t}) \leq R^{*}(h_{1}^{*},h_{2}^{*},t^{*}) + \frac{1}{m} \left( \lambda_{1} A_{1}^{*} + \lambda_{2} A_{2}^{*} \right)
\\ + \frac{c_{1} B^{2}}{\sqrt{m}} \bigg( 
\sqrt{\ln(1/\delta)} + (d+k)^{c_{2} (d+k)} + (\ln( \max\{A_{1}^{*}, \hat{A}_{1}\} \sqrt{m} ))^{c_{2} d} 
\\+ (\ln( \max\{A_{2}^{*},\hat{A}_{2}\} \sqrt{m} ))^{c_{2}(d+k)} \bigg),
\end{multline*}
supposing $d \geq 2$ to simplify the expression.
This expression is exponential in the dimension (reflecting the well-known ``curse of dimensionality'' 
commonly observed in nonparametric regression), but may be decreasing in $m$,
at a rate determined by the $A_{i}^{*}$ and $\hat{A}_{i}$ values
(which themselves can be analyzed in terms of characterizations of the smoothness of the regression function).

\bibliographystyle{plain}
\bibliography{feature-drift} 

\end{document}